\relax
\documentclass[letterpaper]{article} \usepackage{aaai22-no-copyright}  \usepackage{times}  \usepackage{helvet}  \usepackage{courier}  \usepackage[hyphens]{url}  \usepackage{graphicx}    \usepackage{natbib}  \usepackage{caption} \DeclareCaptionStyle{ruled}{labelfont=normalfont,labelsep=colon,strut=off}     

\title{The Partially Observable History Process}
\author{
Dustin Morrill, \textsuperscript{\rm $\dagger$}
    Amy R. Greenwald, \textsuperscript{\rm $\ddagger$}
    Michael Bowling\textsuperscript{\rm $\dagger$}
}
\affiliations{
    \textsuperscript{\rm $\dagger$}University of Alberta; Amii. Edmonton, Alberta, Canada\\
    \textsuperscript{\rm $\ddagger$}Brown University, United States\\

morrill@ualberta.ca,
    amy\_greenwald@brown.edu,
    mbowling@ualberta.ca
}

\usepackage{mathtools}

\DeclarePairedDelimiter{\abs}{\lvert}{\rvert}

\DeclarePairedDelimiter{\subex}{(}{)}
\DeclarePairedDelimiter{\subblock}{[}{]}
\DeclarePairedDelimiter{\tuple}{(}{)}
\DeclarePairedDelimiter{\set}{\{}{\}}

\usepackage{bbold}
\usepackage{bm}

\newcommand{\Simplex}{\Delta}
\newcommand{\simplex}{\Simplex}

\newcommand{\bs}[1]{\bm{#1}}

\newcommand{\expectation}{\mathbb{E}}
\newcommand{\E}{\expectation}
\newcommand{\probability}{\mathbb{P}}
\newcommand{\Prob}{\probability}

\newcommand{\smallo}[1]{\operatorname{o}\subex{#1}}

\newcommand{\ind}[1]{\mathbb{1}\set*{#1}}
\newcommand{\given}{\,|\,}

\newcommand{\where}{\;|\;}

\newcommand{\PureStratSet}{\mathcal{X}}
\newcommand{\PureStrategySet}{\PureStratSet}
\newcommand{\pureStrat}{x}

\newcommand{\utility}{\upsilon}

\newcommand{\StrategySet}{\Pi}
\newcommand{\strategy}{\policy}
\newcommand{\strat}{\strategy}

\newcommand{\recDist}{\mu}

\newcommand{\Actions}{\mathcal{A}}

\newcommand{\regret}{\rho}

\newcommand{\supReward}{U}
\newcommand{\maxReward}{\supReward}

\newcommand{\SWAP}{\textsc{sw}}

\newcommand{\infoSet}{I}
\newcommand{\InfoSets}{\mathcal{I}}

\newcommand{\chance}{c}

\newcommand{\Histories}{\mathcal{H}}
\newcommand{\TerminalHistories}{\mathcal{Z}}

\newcommand{\playerChoice}{p}
\usepackage{amssymb}
\newcommand{\emptyHistory}{\varnothing}

\newcommand{\cfIv}{v}

\newcommand{\cfv}{\cfIv}

\usepackage{upgreek}

\newcommand{\DevSet}{\Phi}
\newcommand{\dev}{\phi}

\newcommand{\COUNTERFACTUAL}{\textsc{cf}}
\newcommand{\CF}{\COUNTERFACTUAL}

\newcommand{\policy}{\pi}
\newcommand{\reward}{r}

\newcommand{\Pohp}{\mathcal{G}}

\newcommand{\RewardFn}{\reward}

\newcommand{\ObservationFn}{\omega}

\newcommand{\ObservationSet}{\mathcal{O}}
\newcommand{\DaimonStratSet}{\Sigma}

\newcommand{\RandomReturn}{G}
\newcommand{\PureDaimonStratSet}{D}
\newcommand{\pureDaimonStrat}{d}

\newcommand{\InfoStateSet}{\mathcal{S}}

\newcommand{\infoState}{s}

\newcommand{\daimonStrat}{\sigma}

\newcommand{\updateFn}{u}
\newcommand{\initialState}{\infoState_{\emptyHistory}}

\newcommand{\belief}{\xi}
\newcommand{\agentActsFirst}{\iota}

\usepackage{xspace}

\makeatletter
\DeclareRobustCommand\onedot{\futurelet\@let@token\@onedot}
\def\@onedot{\ifx\@let@token.\else.\null\fi\xspace}

\def\eg/{\emph{e.g}\onedot} \def\Eg/{\emph{E.g}\onedot}
\def\ie/{\emph{i.e}\onedot} \def\Ie/{\emph{I.e}\onedot}
\def\cf/{\emph{c.f}\onedot} \def\Cf/{\emph{C.f}\onedot}
\def\vs/{\emph{vs}\onedot} \def\Vs/{\emph{Vs}\onedot}
\def\etc/{\emph{etc}\onedot}
\def\wrt/{with respect to} \def\dof/{d.o.f\onedot}
\def\etal/{\emph{et al}\onedot}
\def\viceversa/{\emph{vice-versa}}
\def\ow/{\emph{o.w}\onedot}
\def\whp/{w.h.p\onedot}
\def\apriori/{\emph{a priori}} \def\Apriori/{\emph{A priori}}
\def\ala/{\`{a} la}

\def\naive/{na\"{\i}ve} \def\Naive/{Na\"{\i}ve}
\def\rmPlus/{regret matching\textsuperscript{+}}
\def\rrmPlus/{RRM\textsuperscript{+}}
\def\rcfrPlus/{RCFR\textsuperscript{+}}
\def\cfrPlus/{CFR\textsuperscript{+}}
\@ifdefinable{\Politex/}{\def\Politex/{\textsc{Politex}}}

\def\NashConv/{\textsc{NashConv}}
\def\NashConvAUC/{$\overline{\textsc{NashConv}}$}

\def\heads/{\textsc{heads}}
\def\tails/{\textsc{tails}}
\def\even/{\textsc{even}}
\def\odd/{\textsc{odd}}

\makeatother

\def\efceFootnote/{3}
 \makeatletter
\newcommand\safeIncCounter[1]{\@ifundefined{c@#1}{\newcounter{#1}\stepcounter{#1}}{\stepcounter{#1}}}
\makeatother

\usepackage{xargs}

\usepackage[capitalize]{cleveref}
\usepackage{xcolor}
\definecolor{offWhite}{RGB}{240,240,240}
\definecolor{grey}{RGB}{180,180,180}
\definecolor{darkgreen}{RGB}{0,125,0}
\definecolor{lime}{RGB}{255,200,0}

\definecolor{amiiBlue}{RGB}{16,72,118}
\definecolor{amiiPink}{RGB}{241,97,119}
\definecolor{amiiYellow}{RGB}{248,209,109}
\definecolor{amiiPurple}{RGB}{123,105,145}

\makeatletter
\ifdefined\algorithmic
  
\else
  \usepackage{algorithm}
\usepackage[noend]{algpseudocode}
\fi
\makeatother

\newcommand{\Input}{\State \textbf{Input:}}

\usepackage{amsmath, amssymb, amsfonts, amsthm}
\makeatletter
\ifdefined\theorem
\else
  \newtheorem{theorem}{Theorem}
\fi
\ifdefined\lemma
\else
  \newtheorem{lemma}{Lemma}
\fi
\ifdefined\corollary
\else
  \newtheorem{corollary}{Corollary}
\fi
\ifdefined\definition
\else
  \newtheorem{definition}{Definition}
\fi
\ifdefined\proposition
\else
  \newtheorem{proposition}{Proposition}
\fi
\makeatother
 \usepackage[textwidth=2in,colorinlistoftodos,prependcaption,textsize=footnotesize]{todonotes}
\expandafter\newif\csname ifGin@setpagesize\endcsname
\newcommand{\todonote}[4][inline]{\safeIncCounter{#2NoteCounter}
  \todo[color=offWhite,bordercolor=#3,linecolor=#3,#1]{\textbf{\uppercase{#2}$_{\arabic{#2NoteCounter}}$:}~#4}}

\usepackage[normalem]{ulem}

\newcommand{\replaced}[3]{\def\counterPrefix{#1}
  \def\arrowMarker{#2}
  \def\replacedText{#3}
  \todo[color=offWhite,bordercolor=red,inline]{$\bs{\arrowMarker}$ \textbf{Replaced (\arabic{Replaced\counterPrefix{}NoteCounter})} \replacedText }}

\newcommand{\replacedStart}[2]{\def\user{#1}
  \def\text{#2}
  \safeIncCounter{Replaced#1NoteCounter}\replaced{\user}{\downarrow}{\text}}
\newcommand{\replacedEnd}[1]{\def\user{#1}
  \replaced{\user}{\uparrow}{}}

\newcommand{\issue}[3]{\todo[color=black,inline]{\textcolor{white}{$\bs{#2}$ \textbf{Issue \##3} (Part \arabic{Issue#3NoteCounter}) #1}}}

\newcommand{\issueChangeStart}[2][]{\safeIncCounter{Issue#2NoteCounter}\issue{#1}{\downarrow}{#2}}
\newcommand{\issueChangeEnd}[2][]{\issue{#1}{\uparrow}{#2}}

 \renewcommand{\todonote}[4][inline]{\ignorespaces}
\renewcommand{\issueChangeStart}[2][]{\ignorespaces}
\renewcommand{\issueChangeEnd}[2][]{\ignorespaces}
\renewcommand{\replacedStart}[2][]{\ignorespaces}
\renewcommand{\replacedEnd}[1][]{\ignorespaces}
 \usepackage{tikz}
\usetikzlibrary{
    shapes.geometric,
    arrows,
    arrows.meta,
    graphs,
    graphs.standard,
    calc,
    positioning,
    intersections,
    chains,
    matrix,
    shapes.misc,
    backgrounds,
    overlay-beamer-styles,
    decorations.pathmorphing,
    decorations.pathreplacing,
    shapes.symbols,
    trees
}

\def\cmToEmFactor{2.3710630158366}

\tikzset{
  influenceArrow/.style={>=Triangle, -{>[scale=#1]}},
  influenceArrow/.default=0.4
}

\tikzset{node distance=0cm,inner sep=0cm}
\tikzset{braceDecoration/.style={thick,decorate,decoration={brace,#1}}}

\tikzstyle{devColor} = [draw=red]
\tikzstyle{dev} = [devColor, very thick]

\tikzstyle{followColor} = [draw=black]
\tikzstyle{follow} = [followColor, very thick]

\tikzstyle{infoColor} = [draw=cyan]
\tikzstyle{info} = [infoColor, very thick]
\tikzstyle{infoArrow} = [influenceArrow, thin, infoColor]

\tikzstyle{alt} = [draw=grey]
\tikzstyle{zeroProb} = [draw=grey]
\tikzstyle{rec} = [draw=black, densely dashed, very thick]
\def\stateMinimumRadius{0.18cm}
\tikzstyle{state} = [circle, draw=black, minimum size=2*\stateMinimumRadius, inner sep=0.5mm, fill=white]
\tikzstyle{util} = [inner sep=1mm]
\tikzstyle{behaveLabel} = [text width=2cm]
\tikzstyle{actionArrow} = [thick, >=Stealth, -{>[scale=0.7]}]
\tikzstyle{showPoint} = [shape=circle, fill=black, minimum size=0.3em]

 \usepackage{xargs}

\tikzstyle{conceptShape} = [rectangle, rounded corners, text centered, inner sep=0.2cm, fill=white, draw]
\tikzset{
    timeArrow/.style={thick, >=Stealth, -{>[scale=#1]}},
    timeArrow/.default={0.7}
}
\tikzset{
    flowArrow/.style={thick, >=Triangle, -{>[scale=#1]}},
    flowArrow/.default={0.7}
}
\tikzset{
    pointerArrow/.style={thick, >=Straight Barb, -{>[scale=#1]}},
    pointerArrow/.default={0.7}
}

\tikzstyle{colorBox} = [inner sep=\cmToEmFactor*0.1em, rounded corners, text=white, font=\footnotesize]
\tikzstyle{colorArrow} = [thick, rounded corners, pointerArrow]
\tikzstyle{colorOutline} = [ultra thick, rounded corners]

\newcommandx{\labeledSeq}[7][1={anchor=west}, 2={anchor=north west}]{
    \node(#3)
        [#2,yshift=-0.1em]
        at (#5)
        {#7};
    \node(#3Desc)
        [#1]
        at (#4 |- #3)
        {#6};
}

\newcommandx{\binTreeChildCoords}[3][1=0.9cm, 2=-0.35cm]{
  \coordinate(#3LCoord) at ($(#3.south west)+(-#1,#2)$);
  \coordinate(#3RCoord) at ($(#3.south east)+(#1,#2)$);
}

\newcommandx{\joinHistoriesInInfoSet}[3][1=]{
    \draw
        [ultra thick, rounded corners, #1]
        (#2.north west)
        rectangle
        (#3.south east);
}

\tikzset{text color/.style=normal text.fg}
\tikzset{
    bullets/.pic={
        \node(#1) {$\bullet \quad \bullet \quad \bullet$};
    }
}
 \let\cite\citep

\newcommand{\parencite}{\citep}
\newcommand{\textcite}{\citet}

\begin{document}
\maketitle
\begin{abstract}
  We introduce the partially observable history process (POHP) formalism for reinforcement learning.
  POHP centers around the actions and observations of a single agent and abstracts away the presence of other players without reducing them to stochastic processes.
  Our formalism provides a streamlined interface for designing algorithms that defy categorization as exclusively single or multi-agent, and for developing theory that applies across these domains.
  We show how the POHP formalism unifies traditional models including the Markov decision process, the Markov game, the extensive-form game, and their partially observable extensions, without introducing burdensome technical machinery or violating the philosophical underpinnings of reinforcement learning.
  We illustrate the utility of our formalism by concisely exploring observable sequential rationality, examining some theoretical properties of general immediate regret minimization, and generalizing the extensive-form regret minimization (EFR) algorithm.
\end{abstract}

\section{Introduction}

We develop the partially observable history process (POHP) that embodies the philosophical aspects of reinforcement learning.
The formalism uses a few elementary mechanisms to analyze a single agent that makes observations and takes actions.
The agent is responsible for managing their own representation of an environment that is, by default, massively more complicated than themselves.
They are also responsible for and capable of evaluating themselves against goals set by their designer or themselves.
To canonize the central role of the agent, our formalism abstracts away any other players without nullifying their agency.

The individual components of the POHP formalism are taken from two sequential decision-making frameworks, the extensive-form game (EFG)~\parencite{Kuhn53} and the partially observable Markov decision process (POMDP)~\parencite{smallwood1973pomdp}, along with a repeated game framework, the online decision process (see, \eg/, \textcite{greenwald2006bounds}).
The result is a sequential decision-making formalism that is conceptually simpler than either of its two sequential decision-making progenitors.
Other general formalisms such as the partially observable stochastic game~\parencite{hansen2004posg}, turn-taking partially-observable Markov game~\parencite{Greenwald2017ttpomg}, and factored observation stochastic games~\parencite{kovavrik2019fosg} bring with them unnecessary complications for agent-centric reinforcement learning.
The sequential decision-making setting presented by \textcite{farina2019ocoForSeqDps} shares spiritual similarities but it represents a less radical departure from the EFG model.
\textcite{srinivasan2018actor}'s presentation of the EFG model using reinforcement learning and Markov decision process (MDP) terminology had a substantial influence on this work.

The POHP model is not meant replace any of these established formalisms, but rather fill a particular niche.
We recommend our formalism in two cases: (i) when modeling a problem as an MDP would ignore imperfect information or the presence of other players, or (ii) when the extra structure of a more established model is unnecessary or burdensome.

\section{Partially Observable History Process}

\begin{figure}[t]
  \centering
  \begin{tikzpicture}[node distance=0cm,
  inner sep=0.1em,
  eastLabel/.style={xshift=-0.2em, font=\scriptsize, anchor=east},
  northEastLabel/.style={eastLabel, yshift=0.3em}]

  \coordinate(historyLineY) at (0, 2);
  \coordinate(leftEdge) at (-4.15, 0);

\node(H)
    [anchor=west]
    at ($(leftEdge |- historyLineY)+(3em, 0)$)
    {$H$};
  \node(HA)
    [anchor=west]
    at ($(H.east)+(4em, 0)$)
    {$HA$};
  \node(HAB)
    [anchor=west]
    at ($(HA.east)+(9em, 0)$)
    {$HAB$};
  \coordinate(rightEdge) at ($(HAB)+(3em, 0)$);
  \draw[->]
    (leftEdge |- historyLineY)
    --
    (H);
  \draw[->]
    (H)
    --
    (HA);
  \draw[->]
    (HAB)
    --
    (rightEdge |- HAB);

\node(daimon)
    [anchor=south, yshift=1em]
    at ($(HA.north east)!0.5!(HAB.north west)$)
    {$B$};
  \node(daimon/label)
    [eastLabel]
    at ($(daimon.south)!0.5!(daimon |- HA.north)$)
    {$\sim \daimonStrat(HA)$};
  \draw[->]
    (HA)
    -|
    (daimon);
  \draw[->]
    (HA -| daimon)
    --
    (HAB);
  \draw[->]
    (daimon)
    -|
    (HAB);

\def\observationGap{-1.8em}
  \node(O)
    [anchor=north, yshift=\observationGap]
    at (H.south)
    {$O$};
  \draw[->]
    (H)
    -- node[anchor=east, font=\scriptsize] {$\ObservationFn(H)$}
    (O.north -| H);

  \node(OPrime)
    [anchor=north, yshift=\observationGap]
    at (HAB.south)
    {$O'$};
  \draw[->]
    (HAB)
    -- node[anchor=east, font=\scriptsize] {$\ObservationFn(HAB)$}
    (OPrime.north -| HAB);

\coordinate(agentLineY)
    at ($(O.south -| H)+(0, -0.8em)$);

  \coordinate(labelCenter)
    at ($(agentLineY -| HA)!0.5!(agentLineY -| HAB)$);
  \node
    [anchor=south, font=\footnotesize, yshift=0.2em]
    at (labelCenter)
    {environment};
  \node
    [anchor=north, font=\footnotesize, yshift=-0.2em]
    at (labelCenter)
    {agent};

  \draw
    (O.south -| H)
    --
    (agentLineY);
  \draw
    (OPrime.south -| HAB)
    --
    (agentLineY -| HAB);
  \draw[densely dashed]
    (agentLineY -| leftEdge)
    --
    (agentLineY -| rightEdge);

  \def\stateGap{-3.5em}
  \node(S)
    [yshift=\stateGap]
    at (O |- agentLineY)
    {$S$};
  \node(S/label)
    [northEastLabel]
    at (S.north)
    {$\updateFn_{\ObservationSet}(\bar{S}, O)$};
  \draw[->]
    (agentLineY)
    --
    (S);
  \node(A)
    at ($(S -| HA)!0.5!(agentLineY -| HA)$)
    {$A$};
  \draw[->]
    (leftEdge |- S)
    --
    (S);
  \draw[->]
    (S)
    -|
    (A);
  \node
    [eastLabel]
    at ($(A |- S.north)!0.5!(A.south)$)
    {$\sim \strat(S)$};
  \draw[->]
    (A)
    --
    (HA);

\node(SPrime)
    at (S -| daimon)
    {$S'$};
  \node(SPrime/label)
    [northEastLabel]
    at (SPrime.north)
    {$\updateFn_{\Actions}(S, A)$};
  \draw[->]
    (S -| A)
    --
    (SPrime);
  \draw[->]
    (A)
    -|
    (SPrime);

\node(S2)
    [yshift=\stateGap]
    at (OPrime |- agentLineY)
    {$S''$};
  \node(S2/label)
    [northEastLabel]
    at (S2.north)
    {$\updateFn_{\ObservationSet}(S', O')$};
  \draw[->]
    (SPrime)
    --
    (S2);
  \draw[->]
    (agentLineY -| OPrime)
    --
    (S2);
  \draw[->]
    (S2)
    --
    (S2 -| rightEdge);

\def\returnYShift{-2em}
  \def\returnXShift{1.5em}
  \node(G)
    [yshift=\returnYShift, xshift=\returnXShift]
    at (S)
    {$G$};
  \node(G/label)
    [northEastLabel]
    at (G.north)
    {$\bar{G} + \RewardFn(O)$};
  \draw[->]
    (leftEdge |- G)
    --
    (G);

  \node(GPrime)
    [xshift=\returnXShift]
    at (S2 |- G)
    {$G'$};
  \node(GPrime/label)
    [northEastLabel]
    at (GPrime.north)
    {$G + \RewardFn(O')$};
  \draw[->]
    (G)
    --
    (GPrime);
  \draw[->]
    (GPrime)
    --
    (rightEdge |- GPrime);

\draw[->]
    ($(agentLineY)!0.5!(S)$)
    -|
    (G);
  \draw[->]
    ($(agentLineY -| S2)!0.5!(S2)$)
    -|
    (GPrime);
\end{tikzpicture}   \caption{The evolution of a POHP environment and agent.}
\end{figure}
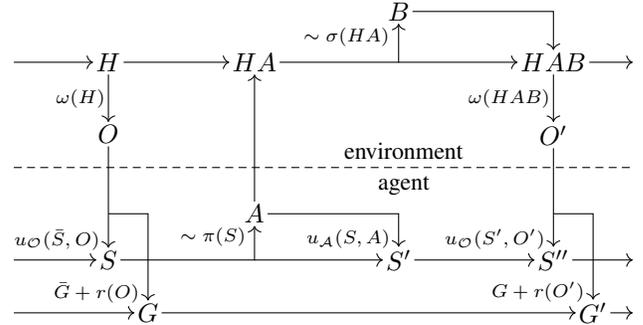

The \emph{partially observable history process} (\emph{POHP}) model begins from the premise that an \emph{agent} observes and influences an \emph{environment}.
We are principally concerned with the design of the agent and how well they navigate the environment.
The environment may change without the agent's input and we attribute these changes to a \emph{daimon}.
Inspired by depictions in Greek mythology, our daimon is an inexplicable force that partially determines the evolution of the environment and shapes the agent's growth.
The concept of a daimon is flexible enough that it can represent an adversary, a teammate, a teacher, chance, or any combination thereof.

\subsection{The Environment and Daimon}

The environment dynamics follow a simple continuing history model.
History in this model refers to a simple ledger that permanently records actions.
Given history $h$ from the set of possible histories $\Histories$ and action $a$, the next history is always $ha \in \Histories$.
The agent and daimon take turns choosing actions until the process terminates and the agent observes the daimon's actions only indirectly.
Histories are partially ordered action strings so we use $h \sqsubset h'$ to denote that $h$ is a predecessor of $h'$, $\abs{h}$ to denote the length of $h$, and use subscripts to reference substrings, \eg/, $h_i$ is the $i^{\text{th}}$ action in $h$ and $h_{\le n}$ is the first $n$ actions of $h$.

A POHP has four core objects:
(i) a Boolean $\agentActsFirst \in \set{0, 1}$ that indicates if the agent acts before the daimon, (ii) a function $\Actions$ that determines the set of legal actions after a given history, (iii) a function $\ObservationFn$ that generates the observation from a set of observations $\ObservationSet$ at a given history immediately following a daimon action, and (iv) a function $\gamma$ that determines the probability that the process continues after a daimon action.
Two more objects are required to describe a POHP in full generality: a set of initial histories $\Histories_{\emptyHistory}$ and a probability distribution over these histories $\belief$.
A POHP's history is initialized with one sampled from $\Histories_{\emptyHistory}$ according to $\belief$, but in most situations $\Histories_{\emptyHistory}$ consists of just the empty history, $\emptyHistory$.

The daimon behaves according to a \emph{behavioral strategy} (also called a \emph{policy}), $\daimonStrat$, that assigns a probability distribution over legal actions to every history where it is the daimon's turn to act,
$\Histories_{\ObservationSet} = \set{h \in \Histories \where \abs{h} \bmod 2 = \agentActsFirst}$.
In history $h \in \Histories_{\ObservationSet}$, the daimon chooses action $B \sim \daimonStrat(h)$ and the history advances to $hB$, at which point the process continues if $\Gamma \sim \gamma(hB)$ and terminates otherwise.
Since the agent waits for an observation in each $h \in \Histories_{\ObservationSet}$, we call them \emph{passive histories}, in contrast to the \emph{active histories}, $\Histories_{\Actions} = \Histories \setminus \Histories_{\ObservationSet}$, in which the agent acts.
We do not ascribe any a priori motivation to the daimon.

\subsection{An Abstract POHP Agent}

The agent need only implement the observation and action interfaces provided by the POHP so there are many ways to construct the agent.
The agent we work with has three conceptual modules: a state of mind, a behavior plan, and goals.
Their state provides the context with which behaviors are chosen to advance them toward their goals.
Formally, we define our POHP agent with a tuple, $\tuple*{\initialState, \updateFn_{\Actions}, \updateFn_{\ObservationSet}, \strat, \RewardFn}$.

\textbf{State of mind.}
A snapshot of the agent's state of mind is given concrete form in their \emph{information state}, which is set to a given initial information state $\initialState$ at the beginning of the POHP.
Information state evolves according to update functions $\updateFn_{\Actions}$ and $\updateFn_{\ObservationSet}$, which describe how actions and observations are processed, respectively.\footnote{Splitting information state updates into action and observation specific updates has a few benefits: (i) it avoids treating the initial observation update in POHPs where the daimon acts first as a special case, (ii) it allows us to refer to the agent's state of mind immediately after choosing an action, which we make use of in our reduction to Markov decision processes and in intermediate proof steps, and (iii) it allows the agent to forget the action they just chose without waiting for an observation, which could be useful in some applications involving asynchronous processing.
}
After the agent chooses action $a$ in information state $\infoState$, their information state is updated to $\infoState' = \updateFn_{\Actions}(\infoState, a)$ while they wait for an observation.
After receiving observation $o$, their information state is updated to $\updateFn_{\ObservationSet}(\infoState', o)$ and the agent chooses another action $a'$ to begin the cycle again.
Ultimately, each history $h \in \Histories$ yields an information state, so we recursively define a unified update function,
\begin{align*}
  \updateFn : h \mapsto
    \begin{cases}
        \initialState & \mbox{ if } h \in \Histories_{\emptyHistory}\\
        \updateFn_{\Actions}\subex*{ \updateFn\subex*{ h_{<\abs{h}} }, h_{\abs{h}} } & \mbox{ if } h_{<\abs{h}} \in \Histories_{\Actions}\\
        \updateFn_{\ObservationSet}\subex*{ \updateFn\subex*{ h_{<\abs{h}} }, \ObservationFn(h) } & \mbox{ o.w.}
    \end{cases}
\end{align*}
Each information state corresponds to an \emph{information set}, $\infoSet(\infoState) = \set{ h \where \updateFn(h) = \infoState }$, which is the set of histories the environment could be in given the agent's information state is $\infoState$.
We denote the set of information states that could ever be generated as $\InfoStateSet$, and we partition them into the passive information states where the agent awaits an observation, $\InfoStateSet_{\ObservationSet}$, and the active information states where the agent acts, $\InfoStateSet_{\Actions}$.
We overload
$\InfoStateSet_{\Actions}(\infoState, a)
  = \set{
    \updateFn_{\ObservationSet}\subex*{
      \updateFn_{\Actions}(\infoState, a),
      \ObservationFn(hab)}
  }_{h \in \infoSet(\infoState), \, b \in \Actions(ha)}$
as the set of child active information states following $\infoState$ and action $a$.

\textbf{Behavior.}
The agent acts by sampling actions from a behavioral strategy, $\strat \in \StrategySet$, where probability distributions over legal actions are assigned to information states.
At history $h$ in active information state $\infoState$, the agent chooses an action by sampling from their \emph{immediate strategy} at $\infoState$,
$\strat(\infoState) \in \simplex(\Actions(h))$,
where $\simplex(\Actions(h))$ is the probability simplex over $\Actions(h)$.
We assume that the agent can always determine the legal actions from their information state so we overload $\Actions(\infoState) = \Actions(h)$ for all $\infoState \in \InfoStateSet$ and $h \in \infoState$.

\textbf{Goals.}
A bounded reward function, $\RewardFn : \ObservationSet \to [-\maxReward, \maxReward]$, provides quantitative feedback to the agent about their progress toward their goals.
The \emph{return} (cumulative reward) that the agent acquires from active history $h \in \Histories_{\Actions}$ is
$\RandomReturn_h(\strat; \daimonStrat)
  = \sum_{i = 1}^{\infty}
    Y_i \RewardFn\big( \ObservationFn(H_i) \big)$,
where the initial history in the trajectory is $H_1 = h$,
the agent's action on each step is $A_i \sim \strat\big( \updateFn(H_i) \big)$,
the daimon's action on each step is $B_i \sim \daimonStrat(H_i A_i)$,
the history is updated as the concatenation $H_{i + 1} = H_i A_i B_i$, and
the continuation indicator is the product $Y_{i + 1} = Y_i \Gamma_i \in \set{0, 1}$ with $Y_1 = 1$ and $\Gamma_i \sim \gamma(H_i)$.

Generally, the agent's goal is to maximize their return.
The fact that the daimon's strategy is unknown and their actions are only partially observed prevents us from immediately formulating this goal as an optimization problem.
Neither can an equilibrium concept be proposed as a solution concept without presupposing incentives and a level of rationality for the daimon.
Hindsight rationality~\parencite{hsr2020}, in contrast, is well suited as a solution concept for POHPs as it focuses on self-improvement grounded in experience and requires no assumptions about the daimon.

However, since no history may ever repeat and repetition is a key requirement of hindsight rationality, we will only consider hindsight rationality in the context of a repeated POHP.
Before each round $t$ begins, the agent chooses strategy $\strat^t$ and the daimon chooses strategy $\daimonStrat^t$.
The POHP plays out according to these strategies, after which the agent receives reward information.
The agent can then compare the returns they achieved with $\strat^t$ with those they could have achieved with alternative behavior.\footnote{The agent may estimate the returns for alternative behavior using importance corrections if this information is not provided explicitly at the end of each round, similarly to how reward functions are estimated in adversarial bandit contexts (see, \eg/, \textcite{lattimore2020bandit}).}
This online decision process is well defined as long as the POHP terminates almost surely so that the agent is unlikely to be stuck in a single round forever.

\subsection{Reach Probabilities}

A derived property of POHPs that we will make use of later is the reach probability.
Consider random history $H$ generated according to agent strategy $\strat$, daimon strategy $\daimonStrat$, and continuation function $\gamma$.
The probability that a given history $h \sqsupseteq h_{\emptyHistory}$ is a prefix of a random history $h \sqsubseteq H$ follows from the chain rule of probability,
$\Prob_{\strat, \daimonStrat}[\bigcup_{\bar{h} \in \Histories_{\emptyHistory}} \bar{h} \sqsubseteq h \sqsubseteq H]
  =
    \prod_{i = \abs{h_{\emptyHistory}} + 1}^{\abs{h}} \Prob_{\strat, \daimonStrat}[h_i \given h_{< i}]$
where
\begin{align*}
  \Prob_{\strat, \daimonStrat}[h_i \given h_{< i}]
    = \begin{cases}
      \gamma(h_{< i}) \strat\subex*{h_i \given \updateFn(h_{< i})} &\mbox{if } h_{< i} \in \Histories_{\Actions}\\
      \daimonStrat(h_i \given h_{< i}) &\mbox{o.w.}
    \end{cases}
\end{align*}
We denote $\Prob_{\strat, \daimonStrat}[h] = \Prob_{\strat, \daimonStrat}[\bigcup_{\bar{h} \in \Histories_{\emptyHistory}} \bar{h} \sqsubseteq h \sqsubseteq H]$ and refer to this quantity as $h$'s \emph{reach probability}.
We can decompose
\begin{align*}
  \Prob_{\strat, \daimonStrat}[h]
    = &\left. \hspace{-1em} \prod_{
        i = \abs{h_{\emptyHistory}} + 1, \, i \bmod 2 = \agentActsFirst
      }^{\abs{h}} \hspace{-1em}
      \gamma(h_{< i - 1}) \daimonStrat(h_i \given h_{< i})
      \right\rbrace \Prob_{\daimonStrat}[h]\\
      &\left. \hspace{-1em} \prod_{i = \abs{h_{\emptyHistory}} + 1, \, i \bmod 2 = 1 - \agentActsFirst}^{\abs{h}} \hspace{-1em} \strat\subex*{h_i \given \updateFn(h_{< i})} \right\rbrace \Prob_{\strat}[h]
\end{align*}
according to the probability that the daimon and agent play their parts of $h$, where we have grouped the continuation probabilities with the daimon's and set $\gamma(h_{< \abs{h_{\emptyHistory}}}) = 1$.

The conditional probability
\begin{align*}
  \Prob_{\strat, \daimonStrat}[h' \given h]
    &= \Prob_{\strat, \daimonStrat}[h \sqsubseteq H, h' \sqsubseteq H] / \Prob_{\strat, \daimonStrat}[h]
\end{align*}
is the probability that history $h' \sqsubseteq H$ given $h \sqsubseteq H$.
If $h'$ and $h$ are unrelated in that
$h' \not\sqsubseteq h \not\sqsubset h'$,
then it is not possible for $H$ to realize both, so the joint probability
$\Prob_{\strat, \daimonStrat}[h, h'] = 0$,
and consequently
$\Prob_{\strat, \daimonStrat}[h' \given h] = 0$.
If $h' \sqsubseteq h$ then $H$ always realizes $h'$ when $h$ is realized, therefore,
$\Prob_{\strat, \daimonStrat}[h, h'] = \Prob_{\strat, \daimonStrat}[h]$
and
$\Prob_{\strat, \daimonStrat}[h' \given h] = 1$.
The last case is $h \sqsubseteq h'$, where
\begin{align*}
  \Prob_{\strat, \daimonStrat}[h' \given h]
    &= \Prob_{\strat, \daimonStrat}[h, h'] / \Prob_{\strat, \daimonStrat}[h]
    = \Prob_{\strat, \daimonStrat}[h'] / \Prob_{\strat, \daimonStrat}[h].
\end{align*}

\section{Representing Traditional Models}
\label{sec:backgroundDecModels}

\subsection{Games}

\begin{algorithm}[tb]
  \caption{The procedure for playing an $N$ player game in POHP-form.}
  \label{alg:efgWithPohps}
  \begin{algorithmic}[1]
  \Input\
    turn function
    $\playerChoice : \Histories \to \set{\chance} \cup \set{i}_{i = 1}^N$,\\
    \quad legal actions function
    $\Actions$,\\
    \quad terminal histories
    $\TerminalHistories \subseteq \Histories$\\
    \quad or continuation function
    $\gamma: \Histories \to \simplex{\set{0, 1}}$,\\
    \quad information partitions
    $\set{\InfoSets_i}_{i \in \set{\chance} \cup \set{j}_{j = 1}^N}$\\
    \quad or observation functions
    $\set{\ObservationFn_i : \Histories \to \InfoStateSet_i}_{i \in \set{\chance} \cup \set{j}_{j = 1}^N}$,\\
    \quad and utility functions
    $\set{\utility_i : \TerminalHistories \to [-\maxReward, \maxReward]}_{i = 1}^N$.
  \For{$i \in \set{\chance} \cup \set{j}_{j = 1}^N$}
    \State $\ObservationFn_i(h) \gets \infoSet$ \textbf{for} $h \in \infoSet \in \InfoSets_i$ \textbf{if} $\ObservationFn_i$ undefined
  \EndFor
  \State $\gamma \gets h \mapsto \ind{h \notin \TerminalHistories}$ \textbf{if} $\gamma$ undefined
  \State $H \gets \emptyHistory$
  \State $\Gamma \gets 1$
  \While{$\Gamma$}
    \State \textbf{send} $\ObservationFn_i(H)$ \textbf{to} player $\playerChoice(H)$
    \State \textbf{receive} $A \in \Actions(H)$ \textbf{from} player $\playerChoice(H)$
    \State $H \gets HA$
    \State \textbf{sample} $\Gamma \sim \gamma(H)$
  \EndWhile
  \For{$i = 1, 2, \ldots, N$}
    \State \textbf{send} $\ObservationFn_i(H) = \tuple{H, \utility_i(H)}$ \textbf{to} player $i$
  \EndFor
\end{algorithmic}
 \end{algorithm}

A \emph{game} is an $N$ player interaction where each player simultaneously chooses a strategy and immediately receives a payoff from a bounded utility function~\parencite{Neumann47}.
There may also be an extra ``chance player'', denoted $\chance$, who ``decides'' chance events like die rolls with strategy $\strat_{\chance}$.
A game described in this way is called a \emph{normal-form game} (\emph{NFG}).

For any given player, $i$, we can represent $i$'s view of the game with a POHP, $\Pohp_i$, where the agent represents $i$ and the daimon represents the other $N - 1$ players and chance in aggregate.
We can also represent chance's view of the game with a POHP where the agent's strategy is fixed to $\strat_{\chance}$.
The histories, action sets, and continuation function across all $N + 1$ of these POHPs are shared but the first turn indicator and observation functions are specific to each player.
The reward functions for each player must also reflect the game's payoffs.
After each player chooses an agent strategy for their POHP, all the POHPs are evaluated together, sharing the same history, and each player receives a return in their POHP that equals their payoff in the game.
Together, the set of POHPs, $\set{\Pohp_i}_{i \in \set{\chance} \cup \set{j}_{j = 1}^N}$, represents what we could call a \emph{POHP-form game}.
See \cref{alg:efgWithPohps} for a programmatic description of how a game can be played out in POHP form.

In each $\Pohp_i$, the daimon's strategy, $\daimonStrat_i$, must reflect those of the other players.
If we have a \emph{turn function} $\playerChoice: \Histories \to \set{\chance} \cup \set{j}_{j = 1}^N$ that determines which player acts after a given history $h$, we can constrain $\daimonStrat_i$ to conform to the agent strategies from the other POHPs as
$\daimonStrat_i(h) = \strat_{\playerChoice(h)}(\updateFn_{\playerChoice(h)}(h))$.

A game described with histories and turns is called an \emph{extensive-form game} (\emph{EFG})~\parencite{Kuhn53}.
Any NFG can be converted into extensive form by serializing each decision.
Of course, players who act later are not allowed to observe previous actions, and this is traditionally specified through information partitions.
Each player, $i$, is assigned a set of information sets as their \emph{information partition}, denoted $\InfoSets_i$.
Typically, EFGs also define a set of \emph{terminal histories}, $\TerminalHistories \subseteq \Histories$, which is constructed so that every history eventually terminates.
As in a NFG, payoffs are given to players upon termination.

Since the POHP and EFG share the same history-based progression, representing an EFG in POHP-form simply requires that information partitions, terminal histories, and utility functions are faithfully reconstructed in the POHP.
If player $i$'s observation function returns its given history's information set in $\InfoSets_i$ and player $i$'s observation update function replaces the current information state with its given observation, then the set of information sets on active information states reproduces $\InfoSets_i$.
Formally, $\set{ \infoSet(\infoState) }_{\infoState \in \InfoStateSet_{i, \Actions}} = \InfoSets_i$.
Constructing the information states for each player in this way ensures that all information partitions are respected.
We can add terminal histories to a POHP by setting $\gamma(h) = \ind{h \notin \TerminalHistories}$.
To respect the EFG's utility function for each player $i$, $\utility_i: \TerminalHistories \to [-\maxReward, \maxReward]$, we set player $i$'s rewards for all observations to zero except those following terminal histories, $z$, at which point $\reward_i(\ObservationFn_i(z)) = \utility_i(z)$.

\subsection{Markov Models}

The POHP model allows agents to construct their own complicated notions of state but forces conceptual simplicity on environment state, \ie/, its history.
However, a popular class of models are Markov models where the environment has a more complicated notion of state and this state evolves according to Markovian dynamics.
That is, there may be many histories that lead to the same environment state and the state on the next step is determined, up to stochasticity, by the current environment state and action.
The transition probabilities must be constant across each history in an environment state, otherwise transitions would depend on past environment states, violating the Markov property.
Agent information states in a POHP may lack the Markov property because the daimon's strategy may depend on the history.

A straightforward way to represent environment states is with the passive information states of the chance player in a POHP-form game.
At each of chance's passive histories $h$, each non-chance player plays an action in turn, which advances the history to $h' = h a_1 \ldots a_N$ where chance updates their information state to
$\infoState_{h'} = \updateFn_{\chance}(h') \in \InfoStateSet_{\chance, \Actions}$.
Chance then chooses which of their passive information states is next by sampling $A_{\chance}$ from $\strat_{\chance}(\infoState_{h'})$, resulting in a transition to
$\infoState_{h'A_{\chance}} = \updateFn_{\chance}(h'A_{\chance}) \in \InfoStateSet_{\chance, \ObservationSet}$.
A Markovian transition between $\infoState_{h}$ and $\infoState_{h'A_{\chance}}$ can be enforced by restricting chance's observation function so that
$\ObservationFn_{\chance}(ha_1 \ldots a_N)
  = \ObservationFn_{\chance}(\bar{h}a_1 \ldots a_N)$
for all joint player actions
$a_1 \ldots a_N$
and histories $\bar{h} \in \infoSet(\infoState_{h})$.
Enforcing this constraint for each history $h$ ensures that if
$\updateFn_{\chance}(\bar{h}) = \infoState_{h}$,
then, given joint player actions $a_1 \ldots a_N$,
\begin{align*}
\updateFn_{\chance}(\bar{h}a_1 \ldots a_N A_{\chance})
  &= \updateFn_{\chance, \Actions}(
    \updateFn_{\chance, \ObservationSet}(
      \updateFn_{\chance}(\bar{h}),
      \ObservationFn_{\chance}(\bar{h} a_1 \ldots a_N)),
    A_{\chance})\\
  &= \updateFn_{\chance, \Actions}(
    \updateFn_{\chance, \ObservationSet}(
      \infoState_{h},
      \ObservationFn_{\chance}(h a_1 \ldots a_N)),
    A_{\chance})\\
  &= \infoState_{h'A_{\chance}}
\end{align*}
with transition probability
$\strat_{\chance}(A_{\chance} \given \infoState_{h'})$
for all $\bar{h}$.

A general Markovian model is the \emph{partially observable Markov game} (\emph{POMG})~\parencite{hansen2004posg}.\footnote{A Markov game is also often called a ``stochastic game'', but a core feature of this model is Markovian transitions, not stochasticity.
  This leads us to prefer the term ``Markov game''.
}
In this model, each environment state represents its own NFG, so all players simultaneously choose an action and receive a reward that depends on the state and joint action selection.
The next state (and thus the next NFG) is determined by the current state and the joint action selection.
Just as we can serialize any NFG by introducing a turn function and selectively hiding actions,
we can serialize a POMG into a \emph{turn-taking POMG} (\emph{TT-POMG})~\parencite{Greenwald2017ttpomg} by serializing each of its component NFGs.

The TT-POMG formalism was developed to convert EFGs into Markov models~\parencite{Greenwald2017ttpomg}, so naturally the POHP-form of a POMG is similar to that of an EFG, except that chance's observations are constrained so that chance's passive information states can play the role of the POMG environment states.
The actions that each player plays are only revealed to the other players after chance's current passive information state transitions to the next one, corresponding to the POMG's next NFG.
Players also receive a reward at this time based on the utilities of the previous NFG.
This model is typically presented as a continuing process with discounting, and we can replicate the same setup by setting the continuation probability $\gamma(h)$ to the discount factor for each of chance's passive histories $h$ and $\gamma(h') = 1$ for all other histories $h'$.

Providing full observability to player $i$ in a POHP-form POMG is simply a matter of adding the constraint that chance's passive information states are isomorphic to player $i$'s active information states.
That is, there must be a bijection where chance's information state is
$\infoState' \in \InfoStateSet_{\chance, \ObservationSet}$
whenever player $i$'s information state is $s \in \InfoStateSet_{i, \Actions}$,
and \viceversa/.
Effectively, player $i$'s active information state is always chance's passive information state and thus also the POMG environment state.
One trivial way to enforce this constraint is for player $i$'s observation function to return chance's passive information state, \ie/,
$\ObservationFn_i(ha_1 \ldots a_{i - 1}) = \updateFn_{\chance}(h)$,
and for player $i$'s observation update function to replace their current state with the given observation.
If all players are granted full observability, then a POMG becomes, naturally, a \emph{Markov game}~\parencite{shapley1953stochasticGame}.
Furthermore, a single-player Markov game or POMG reduces to a \emph{Markov decision process} (\emph{MDP}) or \emph{partially observable MDP} (\emph{POMDP})~\parencite{smallwood1973pomdp}, respectively, and this is true when models are represented either in their canonical or POHP-forms.

\section{The Sub-POHP and Learning}

We now describe how sub-POHPs can be constructed in finite-horizon POHPs with timed updates, and show how observable sequential rationality~\parencite{hsr2020} is naturally defined in terms of sub-POHPs.

A POHP has a \emph{finite horizon} if every history eventually terminates deterministically.
We enforce this by selecting a subset of histories, $\TerminalHistories \subseteq \Histories$ where $\gamma(z) = 0$ for all $z \in \TerminalHistories$.
The agent's updates are \emph{timed} as long as the agent's action update function records the number of actions the agent has taken.
A finite horizon and timed updates ensure that the number of histories in each information set is finite and the same information state is never encountered twice before termination.
Thus, the information states are partially ordered and we can write $\infoState \prec \infoState'$ to denote that information state $\infoState$ is a predecessor of $\infoState'$.

\subsection{Beliefs and Realization Weights}

Given that the agent's information state is $\infoState$, how likely is it that the agent is in a particular history $h \in \infoSet(\infoState)$?
Traditionally, this is called the agent's \emph{belief} (about which history they are in) at $\infoState$.
According to Bayes' rule,
$\Prob_{\strat, \daimonStrat}[h \given \infoState]
  = \Prob_{\strat, \daimonStrat}[\infoState \given h] \Prob_{\strat, \daimonStrat}[h] / \Prob_{\strat, \daimonStrat}[\infoState]$.
Since $h \in \infoSet(\infoState)$,
$\Prob_{\strat, \daimonStrat}[\infoState \given h] = 1$.
The agent's information state is $\infoState$ only if the random history $H$ lands in $\infoSet(\infoState)$, so we can describe the event of realizing $\infoState$ as the union of history realization events.
Since we assume the agent's updates are timed, there is at most one prefix of $H$ in $\infoSet(\infoState)$, which means that each $h \sqsubseteq H$ event for $h \in \infoSet(\infoState)$ is disjoint.
The probability of their union is thus the sum
\begin{align*}
  \Prob_{\strat, \daimonStrat}[\infoState]
    = \Prob_{\strat, \daimonStrat}\subblock*{\bigcup_{\substack{
        h \in \infoSet(\infoState),\\
        \bar{h} \in \Histories_{\emptyHistory}
      }} \bar{h} \sqsubseteq h \sqsubseteq H}
    = \sum_{h \in \infoSet(\infoState)} \Prob_{\strat, \daimonStrat}[h].
\end{align*}
The agent's belief at $\infoState$ is $\belief^{\strat, \daimonStrat}_{\infoState}: h \mapsto
  \Prob_{\strat, \daimonStrat}[h] / \Prob_{\strat, \daimonStrat}[\infoState]$.

An assignment of beliefs to each information state is called a \emph{system of beliefs}.
A problem that arises in defining a complete system of beliefs from a given $\strat$--$\daimonStrat$ pair is that some information states may be unrealizable ($\Prob_{\strat, \daimonStrat}[\infoState] = 0$).
Different rationality assumptions lead to different ways of constructing complete belief systems and corresponding notions of equilibria (see, \eg/, \textcite{KrepsWilson82,breitmoser2010beliefsOffThePath,Dekel2015EpistemicGT}).
However, from a hindsight perspective, only realizable information states could have been observed by the agent, and only behavior in realizable states could have impacted the agent's return.
Thus, beliefs at unreachable information states are naturally left undefined.

As a consequence, information state realization probabilities hold special significance in hindsight analysis, as they determine whether or not a state is observable.
More generally, they provide a measure of importance to each information state.
Let $J$ be the random step in the trajectory $\set{H_i}_{i = 1}^{\infty}$ where $\updateFn(H_J) = \infoState$ or $J = \infty$ if information state $\infoState$ is never realized.
The return from $H_1$ can be split as
\begin{align*}
  &\RandomReturn_{H_1}(\strat; \daimonStrat)
    =
      \ind{J < \infty}
          \sum_{i = 1}^{J - 1} Y_i \RewardFn\big( \ObservationFn(H_i) \big)\\
      &\quad+ \ind{J = \infty}
        \sum_{i = 1}^{\infty} Y_i \RewardFn\big( \ObservationFn(H_i) \big)\\
      &\quad+
        \left.\ind{J < \infty}
          \sum_{i = J}^{\infty} Y_i \RewardFn\big( \ObservationFn(H_i) \big)
        \right\} \text{$\infoState$'s contribution.}
\end{align*}
Since $H_J \sim \belief_{\infoState}^{\strat, \daimonStrat}$, the expectation of $\infoState$'s contribution is
the \emph{realization-weighted expected return} from $\infoState$,
\begin{align}
\cfv_{\infoState}(\strat; \daimonStrat)
  =
    \Prob_{\strat, \daimonStrat}[\infoState]
      \E_{H \sim \belief^{\strat, \daimonStrat}_{\infoState}}\subblock*{
        \RandomReturn_H(\strat; \daimonStrat) },
\label{eq:realizationWeightedExpectedReturn}
\end{align}
where $\belief^{\strat, \daimonStrat}_{\infoState}$ is defined arbitrarily if $\Prob_{\strat, \daimonStrat}[\infoState] = 0$.

\subsection{Observable Sequential Rationality}

Here we capitalize on the generality of our POHP definition.
An agent belief can be used as a distribution over initial histories to define a POHP, which in this context we call a \emph{sub-POHP}.
Thus, every realizable information state $\infoState$ admits a sub-POHP where
$\belief^{\strat, \daimonStrat}_{\infoState}$
is the probability distribution over the histories in
$\infoSet(\infoState)$
and the turn indicator is $\ind{\infoState \in \InfoStateSet_{\Actions}}$.

\emph{Sequential rationality} can then be defined as optimal behavior within every sub-POHP with respect to an assignment of beliefs to unrealizable information states.
This definition is equivalent to sequential rationality in a single-player EFG~\parencite{KrepsWilson82}.
\emph{Observable sequential rationality}~\parencite{hsr2020} merely drops the requirement that play must be rational at unrealizable information states.
The key value determining observable sequential rationality is in fact \cref{eq:realizationWeightedExpectedReturn}, the realization-weighted expected return.

As with rationality in normal and extensive-form games, we can generalize the idea of observable sequential rationality to samples from a joint distribution of agent strategy--daimon strategy pairs (traditionally called a \emph{recommendation distribution}) and deviations.
A \emph{deviation} is a transformation that generates alternative agent behavior, \ie/, a function $\dev: \PureStrategySet \to \PureStrategySet$ where $\PureStrategySet$ is the set of \emph{pure strategies} for the agent that play a single action deterministically in every active information state.
We denote the complete set of such transformations, known as the set of \emph{swap deviations}, as $\DevSet^{\SWAP}_{\PureStrategySet}$.
We can now give a generalized definition of observable sequential rationality in a POHP.
\begin{definition}
  \label{def:obs-seq-rationality}
  A recommendation distribution,
  $\recDist \in \simplex(\PureStratSet \times \PureDaimonStratSet)$,
  where $\PureStratSet$ and $\PureDaimonStratSet$ are the sets of pure strategies for the agent and daimon, respectively,
  is observably sequentially rational for the agent with respect to a set of deviations,
  $\DevSet \subseteq \DevSet^{\SWAP}_{\PureStratSet}$,
  if the maximum benefit for every deviation,
  $\dev \in \DevSet$,
  according to the realization-weighted expected return from every information state,
  $\infoState \in \InfoStateSet$,
  is non-positive,
  \begin{align*}
    \E_{(\pureStrat, \pureDaimonStrat) \sim \recDist}\subblock*{
      \cfv_{\infoState}(\dev(\pureStrat); \pureDaimonStrat)
      - \cfv_{\infoState}(\dev_{\prec \infoState}(\pureStrat); \pureDaimonStrat)
    }
    \le 0,
  \end{align*}
  where $\dev_{\prec \infoState}$ is the deviation that applies $\dev$ only before $\infoState$, \ie/,
  $[\dev_{\prec \infoState} \pureStrat](\bar{\infoState}) = [\dev \pureStrat](\bar{\infoState})$
  if $\bar{\infoState} \prec \infoState$ and $\pureStrat(\bar{\infoState})$ otherwise.
\end{definition}
If $\dev$ always deterministically plays to reach $\infoState$, then this definition becomes equivalent to \textcite{hsr2020}'s.
The hindsight analogue to \cref{def:obs-seq-rationality} follows.
\begin{definition}
  \label{def:obs-seq-hindsight-rationality}
  Define the \emph{full regret} from information state $\infoState$ as
  $\regret_{\infoState}(\dev, \strat; \daimonStrat)
    = \cfv_{\infoState}(\dev(\strat); \daimonStrat) - \cfv_{\infoState}(\strat; \daimonStrat)$.
  An agent is observably sequentially hindsight rational if they are a no-full-regret learner in every realizable information state within a given POHP with respect to $\DevSet \subseteq \DevSet^{\SWAP}_{\PureStratSet}$.
  That is, the agent generates for any $T > 0$ a sequence of strategies, $\tuple{ \strat^t }_{t = 1}^T$, where
  $\lim_{T \to \infty}
      \frac{1}{T} \sum_{t = 1}^T \regret_{\infoState}(\dev, \strat^t; \daimonStrat^t) \le 0$
  at each $\infoState$ for each $\dev \in \DevSet$ under any sequence of daimon strategies $\tuple{ \daimonStrat^t }_{t = 1}^T$.
\end{definition}

\subsection{Local Learning}

Consider a local learning problem in a repeated finite-horizon POHP with timed updates based on the realization-weighted expected return at each active information state $\infoState$.
Given a set of deviations, $\DevSet \subseteq \DevSet^{\SWAP}_{\PureStrategySet}$, we can construct a set of truncated deviations,
$\DevSet_{\preceq \infoState} = \set{ \dev_{\preceq \infoState} }_{\dev \in \DevSet}$,
where each deviation in $\DevSet_{\preceq \infoState}$ applies a deviation from $\DevSet$ until after an action has been taken in $\infoState$, at which point the rest of the strategy is left unmodified.
Each truncated deviation represents a way that the agent could play to and in $\infoState$ so a natural local learning problem is for the agent to choose their actions at $\infoState$ so that there is no beneficial truncated deviation.

To apply deviations to the agent's behavioral strategies, notice that sampling an action for each information state under timed updates yields a pure strategy.
Thus, a behavioral strategy defines a probability distribution over the set of pure strategies, $\PureStratSet$.
We overload $\strat : \PureStratSet \to \simplex(\PureStratSet)$ to return the probability of a given pure strategy under behavioral strategy $\strat \in \StrategySet$.
From this perspective, $\strat$ may be called a \emph{mixed strategy}.
The transformation of $\strat$ by deviation $\dev$ is the pushforward measure $\dev(\strat)$ defined pointwise by
$[\dev\strat](\pureStrat') = \sum_{\pureStrat \in \dev^{-1}(\pureStrat')} \strat(\pureStrat)$
for all $\pureStrat' \in \PureStratSet$,
where $\dev^{-1} : \pureStrat' \mapsto \set{ \pureStrat \where \dev(\pureStrat) = \pureStrat'}$ is the pre-image of $\dev$.

The \emph{immediate regret} at information state $\infoState$ for not employing truncated deviation $\dev_{\preceq \infoState}$ is a difference in realization-weighted expected return under $\belief^{\dev_{\prec \infoState}(\strat), \daimonStrat}_{\infoState}$:
\begin{align*}
  &\regret_{\infoState}(\dev_{\preceq \infoState}, \strat; \daimonStrat)\nonumber
    = \cfv_{\infoState}(\dev_{\preceq \infoState}(\strat); \daimonStrat)
      - \cfv_{\infoState}(\dev_{\prec \infoState}(\strat); \daimonStrat)\\
    &= \Prob_{\dev_{\prec \infoState}(\strat), \daimonStrat}[\infoState]
      \E\subblock*{
        \RandomReturn_H(\dev_{\preceq \infoState}(\strat); \daimonStrat) - \RandomReturn_H(\strat; \daimonStrat) }.
\end{align*}
Intuitively, it is the advantage that $\dev_{\preceq \infoState}(\strat)$ has over $\strat$ in $\infoState$ assuming that the agent plays to $\infoState$ according to $\dev_{\prec \infoState}(\strat)$.\footnote{The term ``advantage'' here is chosen deliberately as immediate regret is analogous to advantage in MDPs~\parencite{baird1994advantageWrtOptimalPolicy}, with respect to a given rather than optimal policy (see, \eg/, \textcite{kakade2003sample}).}
Sadly, it can be impossible to prevent information state $\infoState$'s immediate regret with respect to $\DevSet_{\preceq \infoState}$ from growing linearly in a repeated POHP.

\begin{theorem}
  \label{thm:noImmediateRegretLearningIsImpossible}
  An agent with timed updates cannot generally prevent immediate regret from growing linearly in a finite-horizon repeated POHP.
\end{theorem}
\begin{proof}
Consider a two action, two information state POHP where information state $\infoState$ transitions to $\infoState'$ where the reward is $+1$ if the agent chooses the same action in both $\infoState$ and $\infoState'$, and $-1$ otherwise.
The two \emph{external} (constant) deviations, $\dev^{\to 1}$ and $\dev^{\to 2}$, that choose the same actions in both information states always achieve a value of $+1$.
At $\infoState'$, the agent has to choose between achieving value with respect to the play of $\dev^{\to 1}$ or $\dev^{\to 2}$ in $\infoState$.
If the agent chooses action \#1, then
$\cfv_{\infoState}(\dev^{\to 1}_{\prec \infoState}(\strat); \daimonStrat) = +1$
but
$\cfv_{\infoState}(\dev^{\to 2}_{\prec \infoState}(\strat); \daimonStrat) = -1$,
and \viceversa/ if they choose action \#2.
Therefore, the agent minimizes their maximum regret by always playing uniform random and suffering a regret of $+1$ on every round.
\end{proof}

What if we assume a stronger property, \emph{perfect recall}?
Perfect recall requires that every bit of information from every action and observation is encoded in the information state, \eg/, update functions that concatenate the previous information state with the given action or observation.
Agents with perfect recall remember each of their actions and observations.
This ensures that each information state $\infoState'$ is either the initial information state or has a single parent information state $\infoState$, \ie/, $\updateFn(h_{< \abs{h}}) = \infoState$ for each history $h \in \infoSet(\infoState')$.

A perfect-recall agent can only play to reach each history in a given information state, $\infoState$, equally, \ie/,
$\Prob_{\strat}[h] = \Prob_{\strat}[h']$ for all $h, h' \in \infoSet(\infoState)$.
If we define
$\Prob_{\strat}[\infoState] \propto \sum_{h \in \infoSet(\infoState)} \Prob_{\strat}[h]$,
then perfect recall implies that $\Prob_{\strat}[\infoState] = \Prob_{\strat}[h']$ for any history $h' \in \infoSet(\infoState)$.
The probability of realizing $\infoState$ simplifies to
\begin{align*}
  \Prob_{\strat, \daimonStrat}[\infoState]
    = \sum_{h \in \infoSet(\infoState)}
      \Prob_{\strat}[h] \Prob_{\daimonStrat}[h]
    = \Prob_{\strat}[\infoState]
      \sum_{h \in \infoSet(\infoState)}
        \Prob_{\daimonStrat}[h].
\end{align*}
The belief about any history $h \in \infoSet(\infoState)$ then simplifies to
\begin{align*}
  \belief_{\infoState}^{\strat, \daimonStrat}(h)
    = \frac{
      \Prob_{\strat}[h]\Prob_{\daimonStrat}[h]
    }{
      \Prob_{\strat}[\infoState]
      \sum_{h \in \infoSet(\infoState)}
        \Prob_{\daimonStrat}[h]
    }
    = \frac{
      \Prob_{\daimonStrat}[h]
    }{
      \sum_{h \in \infoSet(\infoState)}
        \Prob_{\daimonStrat}[h]
    }.
\end{align*}
The realization-weighted expected return simplifies to
\begin{align*}
  \cfv_{\infoState}(\strat; \daimonStrat)
    &= \Prob_{\strat}[\infoState]
      \underbrace{
        \sum_{h \in \infoSet(\infoState)}
          \Prob_{\daimonStrat}[h]
          \E\subblock*{\RandomReturn_h(\strat; \daimonStrat)}
      }_{\cfv^{\CF}_{\infoState}(\strat; \daimonStrat)}.
\end{align*}
The sum denoted $\cfv^{\CF}_{\infoState}(\strat; \daimonStrat)$ is recognizable as the \emph{counterfactual value}~\parencite{cfr} of $\infoState$,
which does not depend on $\strat$'s play at $\infoState$'s predecessors.
Immediate regret becomes weighted immediate \emph{counterfactual regret},
\begin{align*}
  \regret_{\infoState}(\dev_{\preceq \infoState}, \strat; \daimonStrat)
    = \Prob_{\dev_{\prec \infoState}(\strat)}[\infoState]
      \subex*{
        \cfv^{\CF}_{\infoState}(\dev_{\preceq \infoState}(\strat); \daimonStrat)
        - \cfv^{\CF}_{\infoState}(\strat; \daimonStrat)}.
\end{align*}

Since the counterfactual value function does not depend on $\strat$'s play at $\infoState$'s predecessors, perfect recall avoids the difficulty that leads to \cref{thm:noImmediateRegretLearningIsImpossible} and allows a reduction from minimizing immediate regret to minimizing \emph{time selection regret} in the \emph{prediction with expert advice} setting~\parencite{blum2007time-selection}.
In a repeated POHP where the agent and daimon choose $\strat^t \in \StrategySet$ and $\daimonStrat^t \in \DaimonStratSet$ on each round $t$ the (time dependent) counterfactual value function
$t \mapsto \cfv^{\CF}_{\infoState}(\cdot; \daimonStrat^t)$
fills the role of the prediction-with-expert-advice reward function and the (time dependent) reach probability function
$w_{\infoState, \dev}: t
  \mapsto \Prob_{\dev_{\prec \infoState}(\strat^t)}[\infoState]$
fills the role of a time selection function.
The growth of cumulative immediate regret can therefore be controlled, in principle, to a sublinear rate by deploying \textcite{blum2007time-selection}'s algorithm or time selection regret matching~\parencite{edl2021}.

\subsection{General Immediate Regret Minimization}

The algorithm that applies a no-time-selection-regret algorithm to minimize immediate regret in every active information state generalizes the \emph{extensive-form regret minimization} (\emph{EFR}) algorithm~\parencite{edl2021} in that $\DevSet \subseteq \DevSet^{\SWAP}_{\PureStrategySet}$ may be any set of deviations rather than a set of behavioral deviations.
Just as \textcite{edl2021} shows that EFR is hindsight rational in EFGs, we can prove that general immediate regret minimization achieves the same in POHPs, though now we can easily present this result as a consequence of observable sequential hindsight rationality.
In principle, our generalized algorithm could compete with the set of swap deviations when given this set (or the set of internal deviations) as a parameter argument, however, circular dependencies between immediate strategies at different information states prevents our algorithm from being efficiently implemented with such a deviation set.

Observable sequential hindsight rationality depends on full regret so we relate immediate regret to full regret with two lemmas and conclude with the algorithm's regret bound.
\begin{lemma}
  In a finite-horizon POHP, the realization-weighted expected return of active information state $\infoState$ under perfect recall recursively decomposes as
  \begin{align*}
    \cfv_{\infoState}(\strat; \daimonStrat)
      &=
        \Prob_{\strat}[\infoState] \RewardFn_{\infoState}(\strat; \daimonStrat)
        + \hspace{-2em} \sum_{
          \infoState' \in \bigcup_{a \in \Actions(\infoState)}
            \InfoStateSet_{\Actions}(\infoState, a)
        } \hspace{-2em}
          \cfv_{\infoState'}(\strat; \daimonStrat)
  \end{align*}
  where
  $\RewardFn_{\infoState}(\strat; \daimonStrat)
    = \sum_{h \in \infoSet(\infoState)}
      \Prob_{\daimonStrat}[h]
      \E\subblock*{\RewardFn(\ObservationFn(hAB))}$,
  and expectations are taken over $A \sim \strat(\infoState)$ and $B \sim \daimonStrat(hA)$.
\end{lemma}
\begin{proof}
  The counterfactual value decomposes as
  \begin{align*}
    \cfv^{\CF}_{\infoState}(\strat; \daimonStrat)
      &= \sum_{h \in \infoSet(\infoState)}
        \Prob_{\sigma}[h] \E\subblock*{
          \RewardFn(\ObservationFn(hAB))
          + \Gamma \RandomReturn_{hAB}(\strat; \daimonStrat)
        }\\
      &= \RewardFn_{\infoState}(\strat; \daimonStrat)
        + \E\subblock{\hspace{-1.5em}
          \sum_{h \in \infoSet(\infoState), b \in \Actions(hA)} \hspace{-1.5em}
              \Prob_{\sigma}[hAb] \RandomReturn_{hAb}(\strat; \daimonStrat)
            }\\
      &= \RewardFn_{\infoState}(\strat; \daimonStrat)
        + \E\subblock*{
          \cfv^{\CF}_{\updateFn_{\Actions}(\infoState, A)}(\strat; \daimonStrat)},
  \end{align*}
  where $\Gamma \sim \gamma(h)$.
  Multiplying by the reach weight,
  \begin{align}
    &\cfv_{\infoState}(\strat; \daimonStrat)\nonumber\\
      &= \Prob_{\strat}[\infoState] \RewardFn_{\infoState}(\strat; \daimonStrat)
        + \sum_{a \in \Actions(\infoState)}
          \Prob_{\strat}[\infoState] \strat(a \given \infoState)
          \cfv^{\CF}_{\updateFn_{\Actions}(\infoState, a)}(\strat; \daimonStrat)\nonumber\\
      &= \Prob_{\strat}[\infoState] \RewardFn_{\infoState}(\strat; \daimonStrat)
        + \sum_{a \in \Actions(\infoState)}
          \cfv_{\updateFn_{\Actions}(\infoState, a)}(\strat; \daimonStrat).
    \label{eq:rwerDecomposition}
\shortintertext{Furthermore,}
&\cfv_{\updateFn_{\Actions}(\infoState, a)}(\strat; \daimonStrat)\nonumber\\
      &=
        \hspace{-1.2em} \sum_{\infoState' \in \InfoStateSet_{\Actions}(\infoState, a)}
          \sum_{
            h \in \updateFn_{\Actions}(\infoState, a)
          } \hspace{-1em}
            \ind{\updateFn(h) = \infoState'}
            \Prob_{\strat, \daimonStrat}[h]
            \E \subblock*{
              \RandomReturn_{h}(\strat; \daimonStrat)
            }\nonumber\\
      &=
        \hspace{-1.2em} \sum_{\infoState' \in \InfoStateSet_{\Actions}(\infoState, a)}
          \underbrace{
            \sum_{h' \in \infoSet(\infoState')}
              \Prob_{\strat, \daimonStrat}[h']
              \E \subblock*{
                \RandomReturn_{h'}(\strat; \daimonStrat)
              }
          }_{\cfv_{\infoState'}(\strat; \daimonStrat)}.
      \label{eq:passiveToActiveValue}
  \end{align}
  Substituting \cref{eq:passiveToActiveValue} into \cref{eq:rwerDecomposition} completes the proof.
\end{proof}
\begin{lemma}
  \label{lemma:regretDecomposition}
  In a finite-horizon POHP, the full regret with respect to $\dev \in \DevSet^{\SWAP}_{\PureStrategySet}$ under perfect recall at active information state $\infoState$ recursively decomposes as
  \begin{align*}
    \regret_{\infoState}(\dev, \strat; \daimonStrat)
&=
        \regret_{\infoState}(\dev_{\preceq \infoState}, \strat; \daimonStrat)
        + \hspace{-2.2em} \sum_{
          \infoState' \in \bigcup_{a \in \Actions(\infoState)}
            \InfoStateSet_{\Actions}(\infoState, a)
        } \hspace{-2.2em}
          \regret_{\infoState'}(\dev, \strat; \daimonStrat).
  \end{align*}
\end{lemma}
\begin{proof}
  \begin{align*}
    &\regret_{\infoState}(\dev, \strat; \daimonStrat)\nonumber\\
      &=
        \cfv_{\infoState}(\dev(\strat); \daimonStrat)
        \rlap{$\overbrace{\phantom{- \cfv_{\infoState}(\dev_{\preceq \infoState}(\strat); \daimonStrat)+ \cfv_{\infoState}(\dev_{\preceq \infoState}(\strat); \daimonStrat)\,}}^0$}
        - \cfv_{\infoState}(\dev_{\preceq \infoState}(\strat); \daimonStrat)
        + \underbrace{
          \cfv_{\infoState}(\dev_{\preceq \infoState}(\strat); \daimonStrat)
          - \cfv_{\infoState}(\strat; \daimonStrat)
        }_{\regret_{\infoState}(\dev_{\preceq \infoState}, \strat; \daimonStrat)}\\
      &=
        \regret_{\infoState}(\dev_{\preceq \infoState}, \strat; \daimonStrat)\\
        &\quad+ \underbrace{
          \Prob_{\dev(\strat)}[\infoState] \RewardFn_{\infoState}(\strat; \daimonStrat)
          - \Prob_{\dev_{\preceq \infoState}(\strat)}[\infoState] \RewardFn_{\infoState}(\strat; \daimonStrat)
        }_{0}\\
        &\quad+ \sum_{a \in \Actions(\infoState)}
          \underbrace{
            \cfv_{\updateFn_{\Actions}(\infoState, a)}(\dev(\strat); \daimonStrat)
            - \cfv_{\updateFn_{\Actions}(\infoState, a)}(\dev_{\preceq \infoState}(\strat); \daimonStrat)
          }_{\regret_{\updateFn_{\Actions}(\infoState, a)}(\dev, \strat; \daimonStrat)}.\\
      \shortintertext{Applying \cref{eq:passiveToActiveValue} to sum over active information states,}
      &=
        \regret_{\infoState}(\dev_{\preceq \infoState}, \strat; \daimonStrat)
        + \hspace{-2em} \sum_{
          \infoState' \in \bigcup_{a \in \Actions(\infoState)}
            \InfoStateSet_{\Actions}(\infoState, a)
        } \hspace{-1em}\underbrace{
            \cfv_{\infoState'}(\dev(\strat); \daimonStrat)
            - \cfv_{\infoState'}(\strat; \daimonStrat)
          }_{\regret_{\infoState'}(\dev, \strat; \daimonStrat)}.\qedhere
  \end{align*}
\end{proof}
\begin{theorem}
  If a perfect recall agent's cumulative immediate regret with respect to $\DevSet \subseteq \DevSet^{\SWAP}_{\PureStrategySet}$ at each information state $\infoState$ in a repeated finite-horizon POHP is upper bounded by $f(T) \ge 0$, $f(T) \in \smallo{T}$ after $T$ rounds, then the agent's cumulative full regret at each $\infoState$ is sublinear, upper bounded according to
  $\sum_{t = 1}^T \regret_{\infoState}(\dev, \strat^t; \daimonStrat^t)
    \le \abs{\InfoStateSet_{\infoState, \Actions}} f(T)$,
  where $\InfoStateSet_{\infoState, \Actions} = \set{ \infoState' \in \InfoStateSet_{\Actions} \where \infoState \preceq \infoState' }$ is the active information states in the sub-POHP rooted at $\infoState$.
  Such an agent is therefore observably sequentially hindsight rational with respect to $\DevSet$.
\end{theorem}
\begin{proof}
  Working from each terminal information state where the full and immediate regret are equal toward $\infoState$ at the root of any given sub-POHP, we recursively bound the cumulative full regret at every information state according to \cref{lemma:regretDecomposition}.
  Every active information state adds at most $f(T)$ to the cumulative full regret at $\infoState$ and there are $\abs{\InfoStateSet_{\infoState, \Actions}}$ active information states in $\infoState$'s sub-POHP so the cumulative full regret at $\infoState$ is no more than $\abs{\InfoStateSet_{\infoState, \Actions}} f(T)$.
\end{proof}

\section{Conclusion}

The POHP formalism may be useful in modeling continual learning problems where environments are expansive, unpredictable, and dynamic.
Good performance here demands that the agent continually learns, adapts, and re-evaluates their assumptions.
We suspect that hindsight rationality could serve as the learning objective for such problems if it could be formulated for a single agent lifetime rather than over a repeated POHP.

Our analysis of general immediate regret minimization for POHPs and the impossibility result of \cref{thm:noImmediateRegretLearningIsImpossible} brings up questions about how far this procedure can be generalized.
Regret decomposition is based on a perfect-recall, realization-weighted variant of \textcite{kakade2003sample}'s performance difference lemma (Lemma 5.2.1).
\textcite{ltbc2021} use this to show how the counterfactual regret minimization (CFR)~\parencite{cfr} (EFR with counterfactual deviations~\parencite{edl2021}) can be applied to continuing, discounted MDPs with reward uncertainty.
The POHP formalism can perhaps allow us to better understand when immediate and full regret can be minimized without perfect recall by considering \textcite{lanctot2012no}'s well-formed-game conditions together with \textcite{ltbc2021}'s analysis.

The POHP formalism allows agents to determine their own representation of the environment.
This opens the way to direct discussions and comparisons of representations and updating schemes.
One particular direction that is made natural by the POHP model's action--observation interface is predictive state representations (PSRs)~\parencite{singh2003psr,singh2012psrSystemDynamicsMatrix}.
While PSRs were developed to model Markovian dynamical systems with at most one controller, the POHP model could facilitate an extension to multi-agent settings.

\section*{Acknowledgments}
Dustin Morrill and Michael Bowling are supported by the Alberta Machine Intelligence Institute (Amii), CIFAR, and NSERC.
Amy Greenwald is supported in part by NSF Award CMMI-1761546.
Thanks to Michael's and Csaba Szepesv\'{a}ri's research groups for feedback that helped to refine early versions of the POHP formalism.

\bibliography{references}

\end{document}